\documentclass[letterpaper, 10 pt, conference]{ieeeconf} % DO NOT CHANGE THIS
\IEEEoverridecommandlockouts
\usepackage{times}  % DO NOT CHANGE THIS
\usepackage{helvet}  % DO NOT CHANGE THIS
\usepackage{courier}  % DO NOT CHANGE THIS
\usepackage[hyphens]{url}  % DO NOT CHANGE THIS
\usepackage{graphicx} % DO NOT CHANGE THIS
\usepackage{caption} % DO NOT CHANGE THIS AND DO NOT ADD ANY OPTIONS TO IT
\usepackage{algorithm}
\usepackage{algorithmic}

\usepackage{newfloat}
\usepackage{listings}
\DeclareCaptionStyle{ruled}{labelfont=normalfont,labelsep=colon,strut=off} % DO NOT CHANGE THIS
\floatstyle{ruled}
\newfloat{listing}{tb}{lst}{}
\floatname{listing}{Listing}
 \usepackage{amsmath}
 \usepackage{amssymb}
  
 \usepackage{amsthm}
\usepackage{booktabs}
\usepackage{multirow}
\usepackage{multicol}
\usepackage{subfigure}
\usepackage{xcolor}
\usepackage{tabularx}

\newcommand{\citet}[1]{\cite{#1}} %natbibs's citet clashes with IEEE?

\newtheorem{theorem}{Theorem}%[section]
\newtheorem{lemma}[theorem]{Lemma}
\theoremstyle{remark}
\newtheorem*{remark}{Remark}

\newcommand{\atcresults}[0]{
\begin{figure}[!t]
    \centering
    \includegraphics[width = \columnwidth]{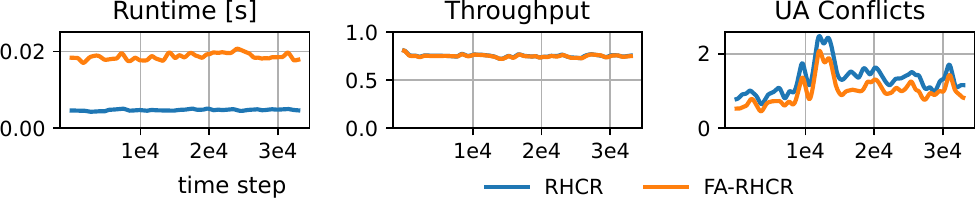}
    \caption{Results of running RHCR on the ATC-map. x-axis is the number of steps, y-axis the respective metric value. RHCR runs for 33000 timesteps with 40 MAPF agents.}
    \label{fig:atc_results}
\end{figure}
}

\newcommand{\numagentcomp}[0]{
\begin{figure}[!t]
    \centering
    \includegraphics[width = \columnwidth]{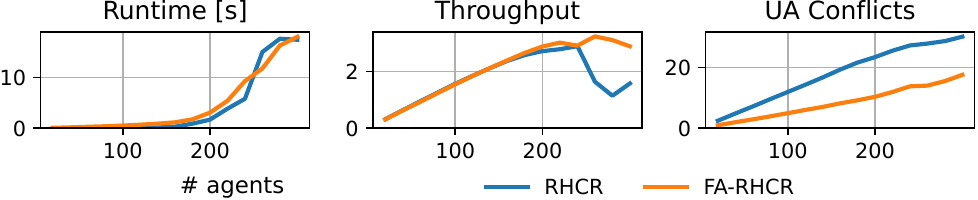}
    \caption{Results of running RHCR on the \textit{den312d}-map. x-axis is the number of MAPF agents, y-axis the respective metric value. UA movement type is \textit{directed}.}
    \label{fig:num_agent_comp}
\end{figure}
}

\newcommand{\windowscomp}[0]{
\begin{figure}[!t]
    \centering
    \includegraphics[width = \columnwidth]{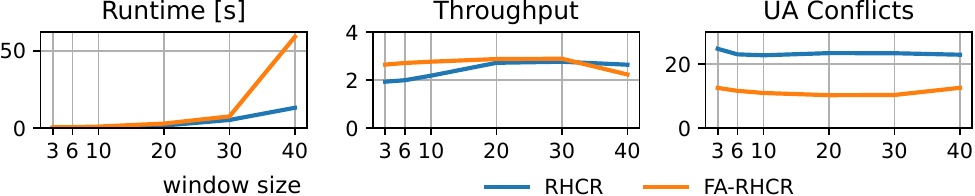}
    \caption{Results of running RHCR with 200 agents on the \textit{den312d}-map. x-axis is the replanning window size of RHCR, y-axis the respective metric value. UA movement type is \textit{directed}.}
    \label{fig:windows_comp}
\end{figure}
}

\newcommand{\warehouse}[0]{
\begin{figure}[!t]
    \centering
    \includegraphics[width = \columnwidth]{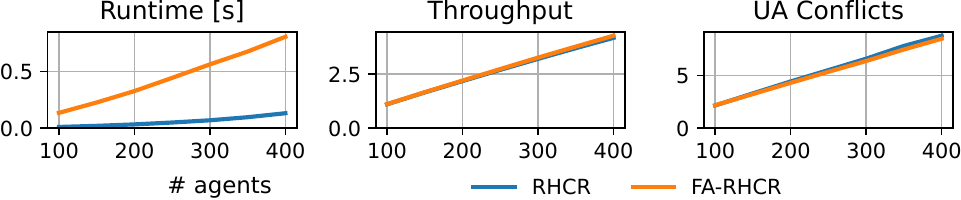}
    \caption{Results of running RHCR on the \textit{warehouse-10-20-10-2-1}-map. x-axis is the number of agents, y-axis the respective metric value. UA movement type is \textit{directed}.}
    \label{fig:warehouse}
\end{figure}
}

\newcommand{\warehousecc}[0]{
\begin{figure}[!t]
    \centering
    \includegraphics[width = \columnwidth]{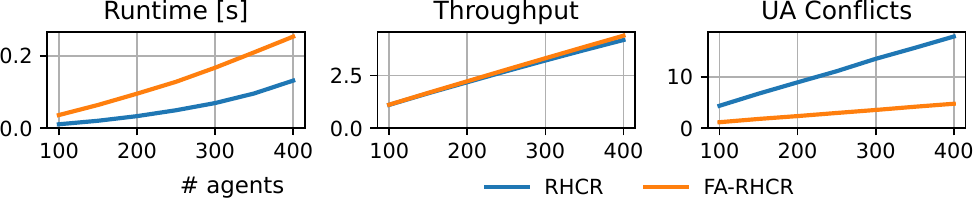}
    \caption{Results of running RHCR on the \textit{warehouse-10-20-10-2-1}-map. x-axis is the number of agents, y-axis the respective metric value. UA movement type is \textit{directed} but enforces highway patters \cite{li2023study}.}
    \label{fig:warehouse_crisscross}
\end{figure}
}

\newcommand{\oneshot}[0]{
\begin{figure*}[!t]
    \centering
    \includegraphics[width = 2\columnwidth]{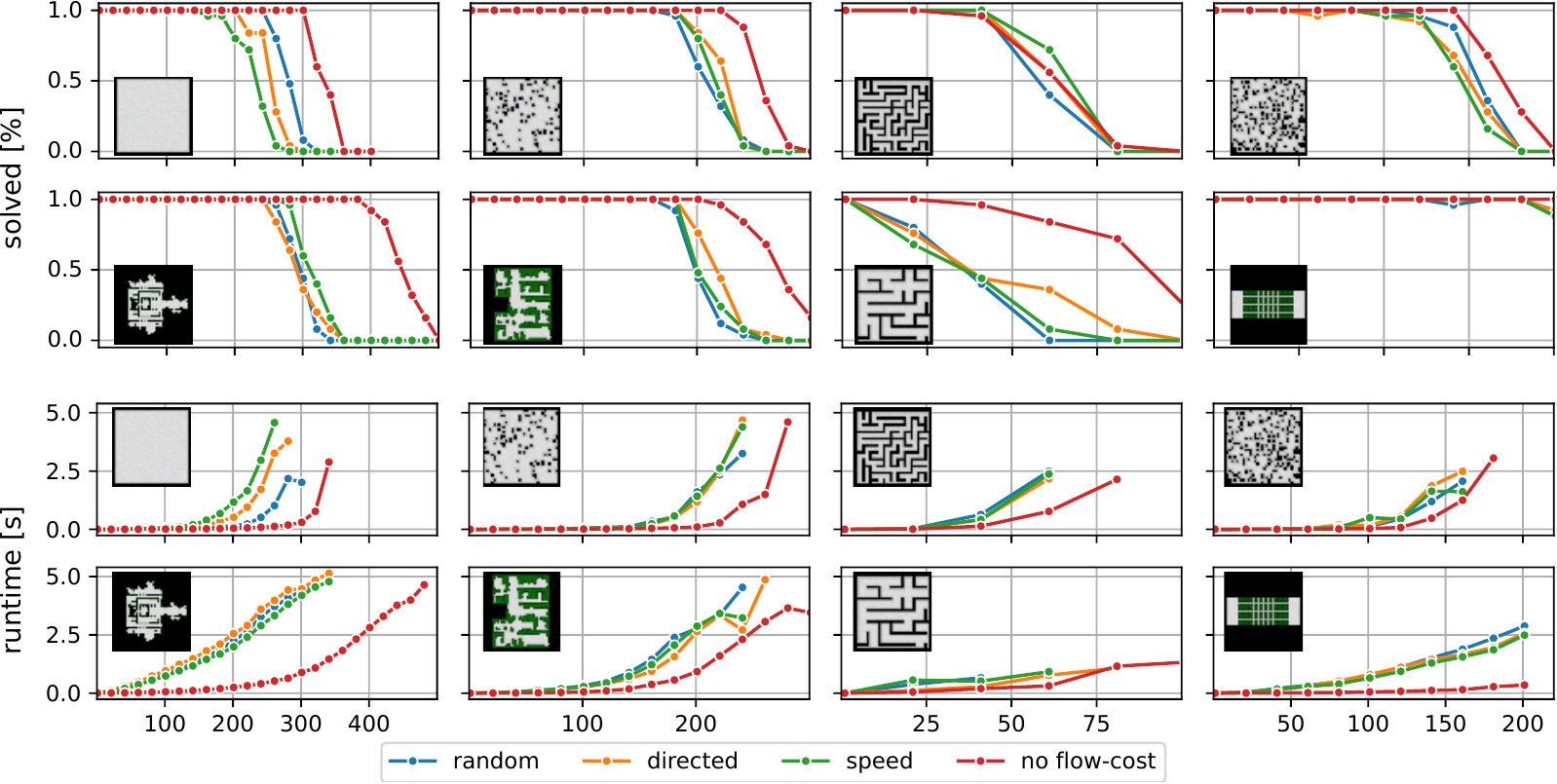}
    \caption{Result of one-shot EECBS on the different benchmark maps (see Table~\ref{table:maptable} for the map names), for different UA movement types (as described in Section~\ref{sec:evaluation}). x-axis is the number of MAPF agents. %y-axis number of instances solved (out of 25) or the runtime respectively. 
    y-axis shows either the percentage (\textbf{top}) of instances solved (out of 25) or the average runtime (\textbf{bottom}). The runtime cutoff is 5 seconds.}
    \label{fig:oneshot}
\end{figure*}
}

\newcommand{\modcost}[0]{
\begin{figure}[!t]
    \centering
    \includegraphics[width = \columnwidth]{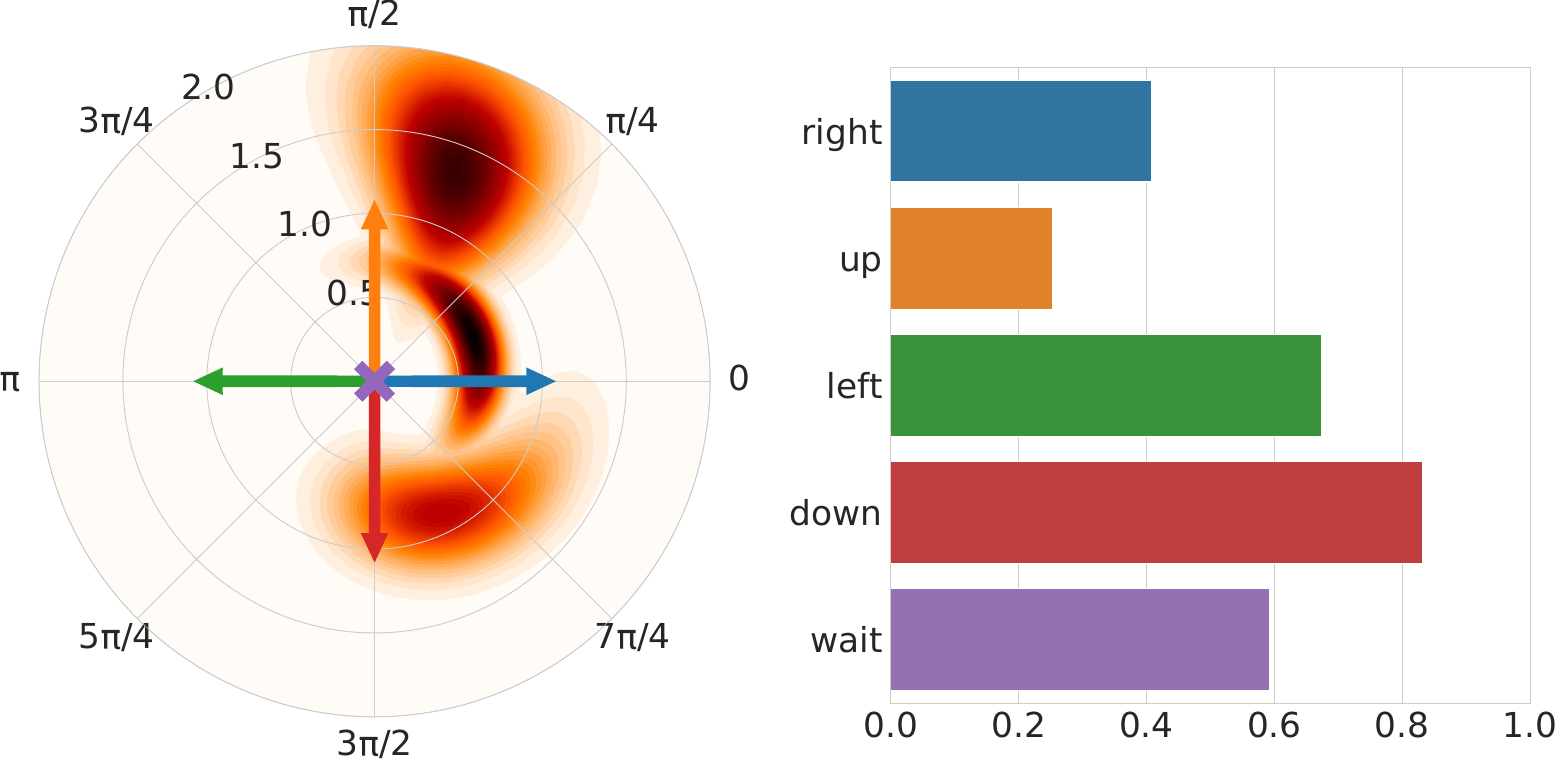}
    \caption{\textbf{Left:} A visualization of an SWGMM, obtained from an MoD, that represents the likelihood of UA movement for a specific location in a map. The colored arrows represent the set of possible actions of a MAPF agent. \textbf{Right:} Flow cost calculated using Eq. \ref{eq:malahanobis_distance} representing the \textit{distance} from an action to the SWGMM.}
    \label{fig:mod_cost}
\end{figure}
}

\newcommand{\modareas}[0]{
\begin{figure}[!t]
    \centering
    \includegraphics[width = \columnwidth]{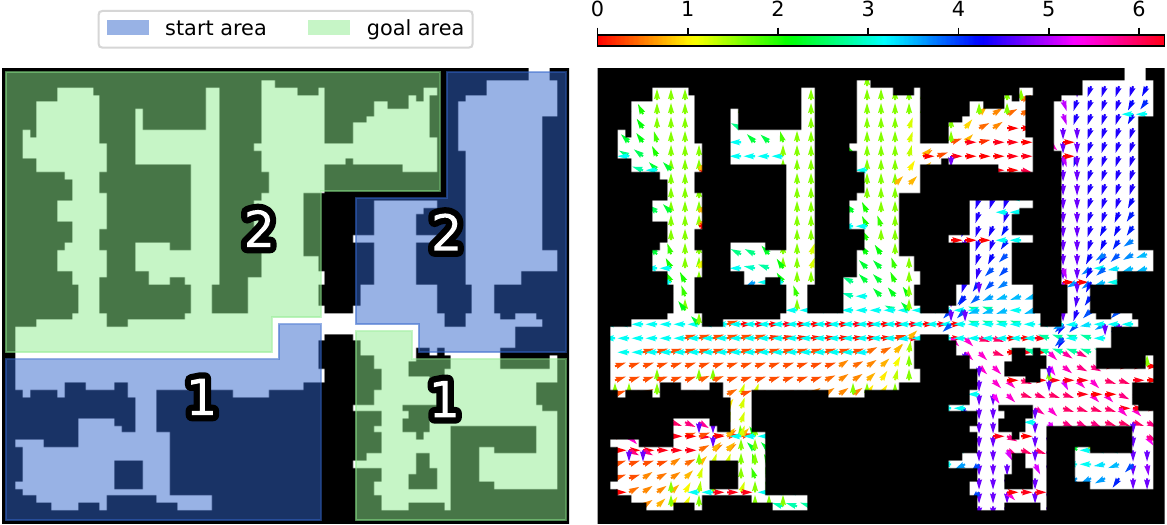}
    \caption{\textbf{Left:} It shows the areas from which the start and goal points are sampled for \textit{directed} UA movement. The numbers indicate the corresponding areas. \textbf{Right:} The corresponding CLiFF-map of dynamics, with the arrows representing the means of the SWGMM mixture components, and color indicates their direction in radians.}
    \label{fig:mod_areas}
\end{figure}
}

\newcommand{\directioncost}[0]{
\begin{figure}[!t]
    \centering
    \includegraphics[width = \columnwidth]{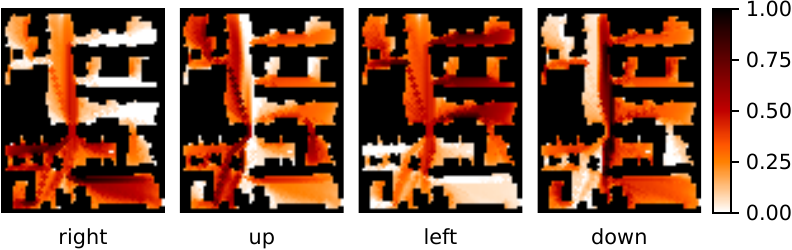}
    \caption{The cost for an edge transition for every location in the \textit{den312d}-map, for each move action respectively. It is computed using the MoD in Figure \ref{fig:mod_areas} and shows how FA-MAPF translates motion patterns into a semantic cost-landscape for the MAPF algorithm.}
    \label{fig:direction_cost}
\end{figure}
}

\newcommand{\maptable}[0]{
\begin{table*}[t]
\small
\centering
\begin{tabularx}{\textwidth}{lll|ll|ll|ll}
\toprule
 & &  & \multicolumn{2}{c}{UA Conflicts} & \multicolumn{2}{c}{Runtime [s]} & \multicolumn{2}{c}{Throughput} \\
 & & MAPF method & RHCR & FA-RHCR & RHCR & FA-RHCR & RHCR & FA-RHCR \\
& map name & movement type &  &  &  &  &  &  \\
\midrule
\multirow[b!]{2}{*}{\includegraphics[width=1cm]{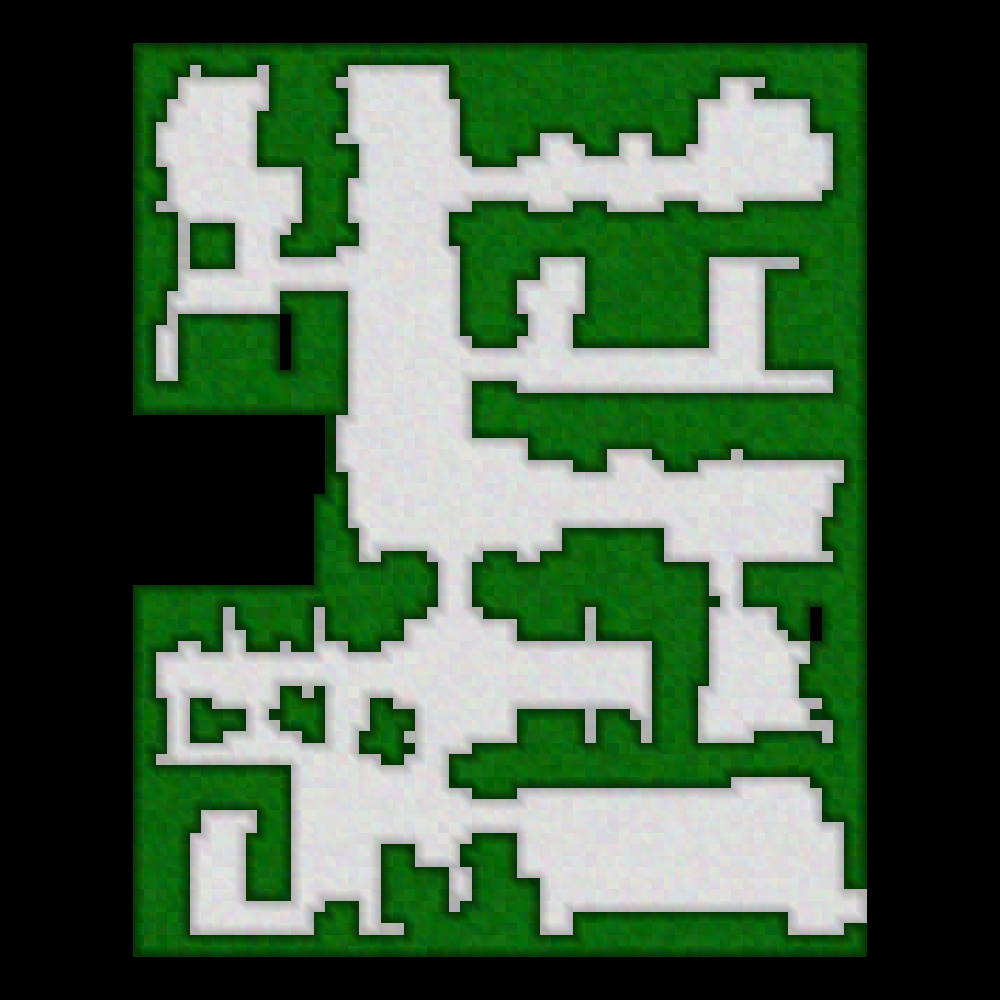}}& den312d & directed & 23.47~$\pm$~3.04 & \textbf{10.52~$\pm$~1.43} & \textbf{1.68~$\pm$~1.74} & 3.22~$\pm$~2.95 & 2.73 & \textbf{2.89} \\
& \multirow[t]{2}{*}{\# agents: 200} & random & 19.97~$\pm$~2.37 & \textbf{15.30~$\pm$~1.97} & 1.68~$\pm$~1.74 & \textbf{1.07~$\pm$~0.64} & 2.73 & \textbf{2.83} \\
&  & speed & 21.55~$\pm$~3.17 & \textbf{15.44~$\pm$~2.31} & \textbf{1.68~$\pm$~1.74} & 7.47~$\pm$~3.82 & \textbf{2.73} & 2.57 \\
\midrule
\multirow[b!]{2}{*}{\includegraphics[width=1cm]{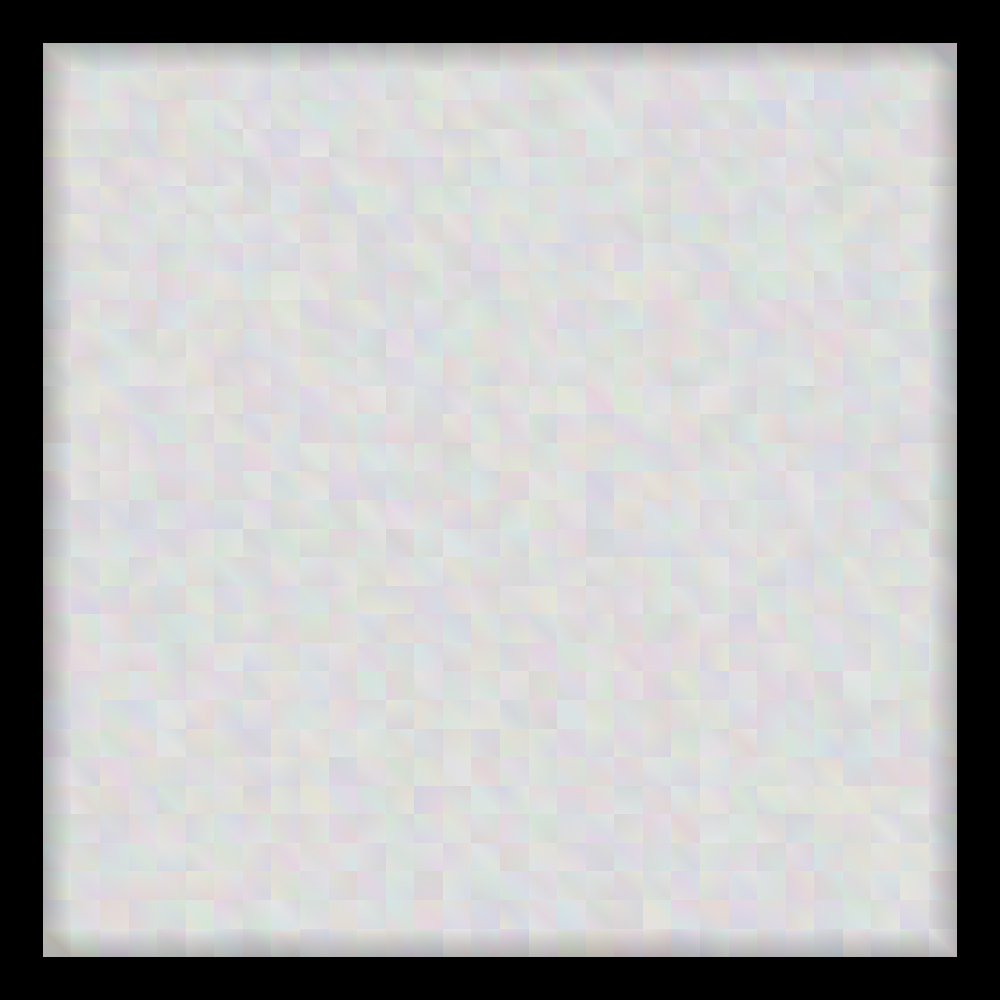}}& empty-32-32 & directed & 10.65~$\pm$~0.92 & \textbf{6.82~$\pm$~0.81} & \textbf{0.37~$\pm$~0.07} & 1.87~$\pm$~0.59 & \textbf{8.28} & 8.27 \\
& \multirow[t]{2}{*}{\# agents: 300} & random & 13.93~$\pm$~1.02 & \textbf{11.99~$\pm$~0.97} & \textbf{0.38~$\pm$~0.08} & 0.52~$\pm$~0.05 & 8.28 & \textbf{8.44} \\
&  & speed & 9.50~$\pm$~0.82 & \textbf{6.52~$\pm$~0.70} & \textbf{0.37~$\pm$~0.07} & 3.73~$\pm$~1.79 & \textbf{8.28} & 8.19 \\
\midrule
\multirow[b!]{2}{*}{\includegraphics[width=1cm]{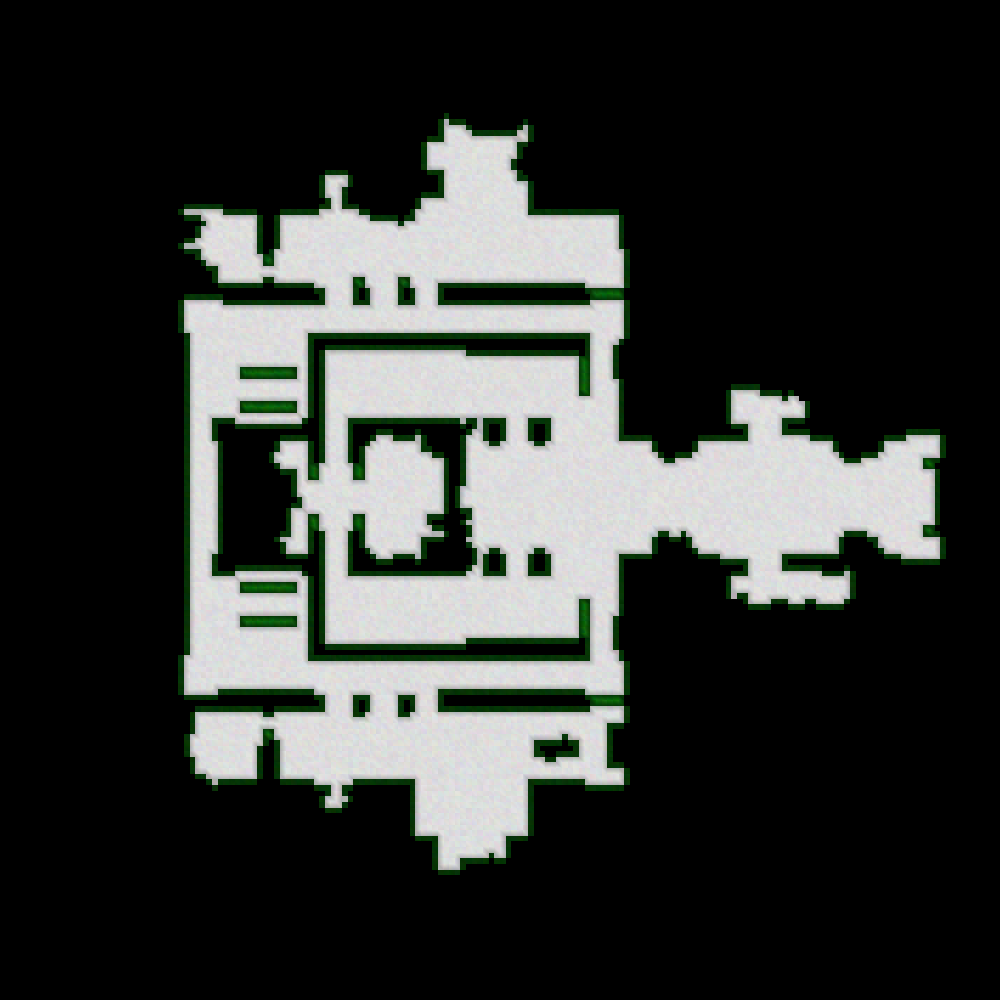}}& ht\_chantry & directed & 21.15~$\pm$~2.86 & \textbf{9.77~$\pm$~1.32} & \textbf{0.60~$\pm$~0.34} & 5.28~$\pm$~2.06 & 4.67 & 4.67 \\
& \multirow[t]{2}{*}{\# agents: 500} & random & 22.19~$\pm$~2.67 & \textbf{16.07~$\pm$~1.88} & \textbf{0.59~$\pm$~0.34} & 2.95~$\pm$~0.62 & 4.67 & \textbf{4.70} \\
&  & speed & 21.17~$\pm$~2.65 & \textbf{10.14~$\pm$~1.25} & \textbf{0.61~$\pm$~0.35} & 6.81~$\pm$~1.76 & \textbf{4.67} & 4.63 \\
\midrule
\multirow[b!]{2}{*}{\includegraphics[width=1cm]{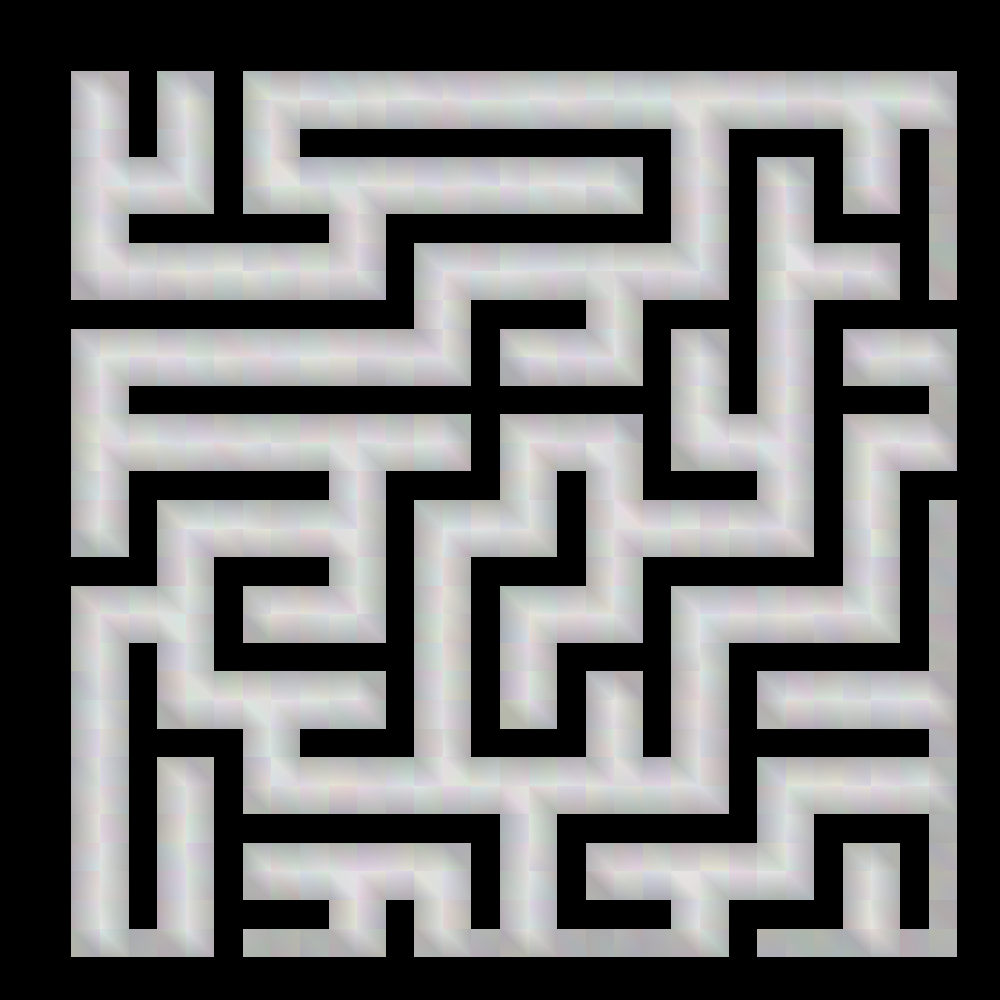}}& maze-32-32-2 & directed & 8.99~$\pm$~1.25 & \textbf{8.55~$\pm$~1.23} & \textbf{0.01~$\pm$~0.01} & 1.30~$\pm$~0.84 & 0.40 & 0.40 \\
& \multirow[t]{2}{*}{\# agents: 30} & random & 7.95~$\pm$~1.02 & \textbf{7.59~$\pm$~1.05} & \textbf{0.01~$\pm$~0.01} & 0.09~$\pm$~0.04 & 0.40 & 0.40 \\
&  & speed & 9.09~$\pm$~1.85 & \textbf{8.50~$\pm$~1.73} & \textbf{0.01~$\pm$~0.01} & 5.20~$\pm$~4.62 & 0.40 & 0.40 \\
\midrule
\multirow[b!]{2}{*}{\includegraphics[width=1cm]{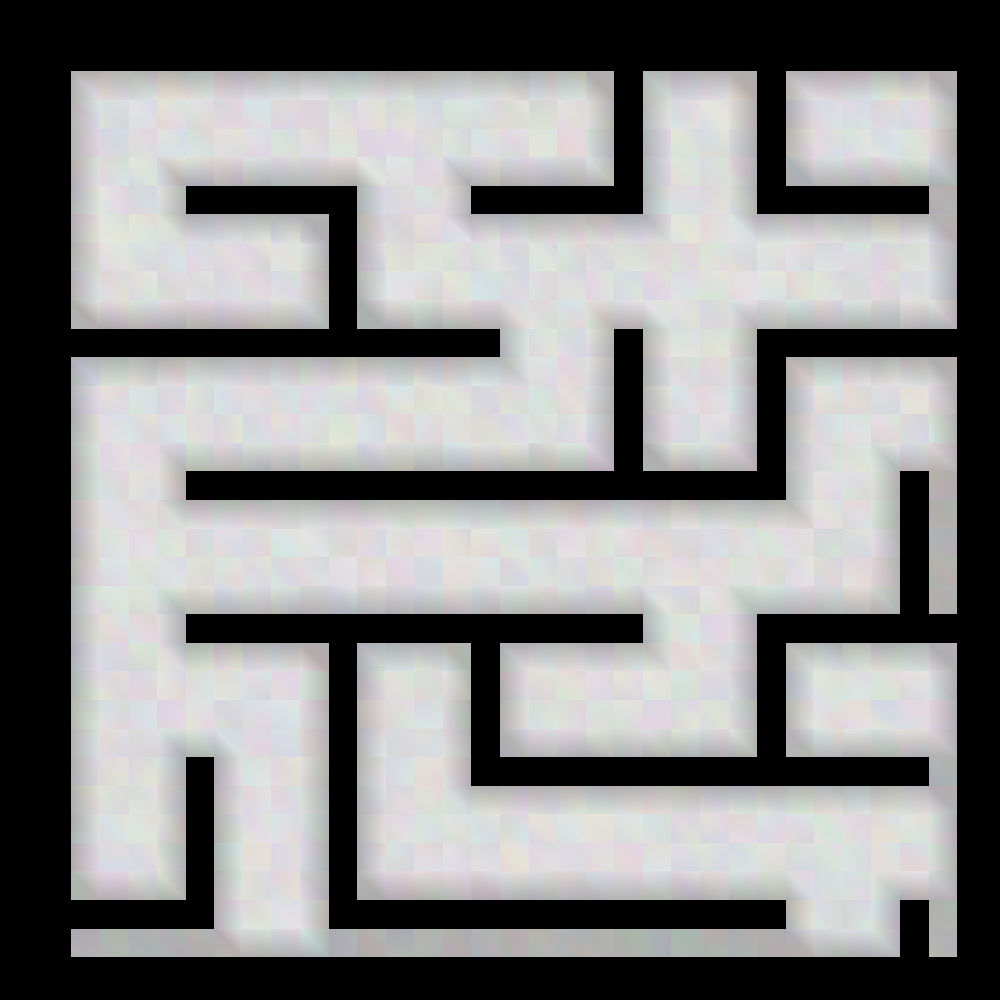}}& maze-32-32-4 & directed & 5.80~$\pm$~1.46 & \textbf{4.76~$\pm$~1.43} & \textbf{0.26~$\pm$~2.00} & 9.19~$\pm$~7.97 & 0.43 & 0.43 \\
& \multirow[t]{2}{*}{\# agents: 30} & random & 5.69~$\pm$~1.24 & \textbf{4.75~$\pm$~1.19} & \textbf{0.26~$\pm$~2.00} & 0.78~$\pm$~2.79 & \textbf{0.43} & 0.42 \\
&  & speed & 5.76~$\pm$~1.85 & \textbf{4.87~$\pm$~1.27} & \textbf{0.26~$\pm$~2.00} & 10.47~$\pm$~7.67 & \textbf{0.43} & 0.42 \\
\midrule
\multirow[b!]{2}{*}{\includegraphics[width=1cm]{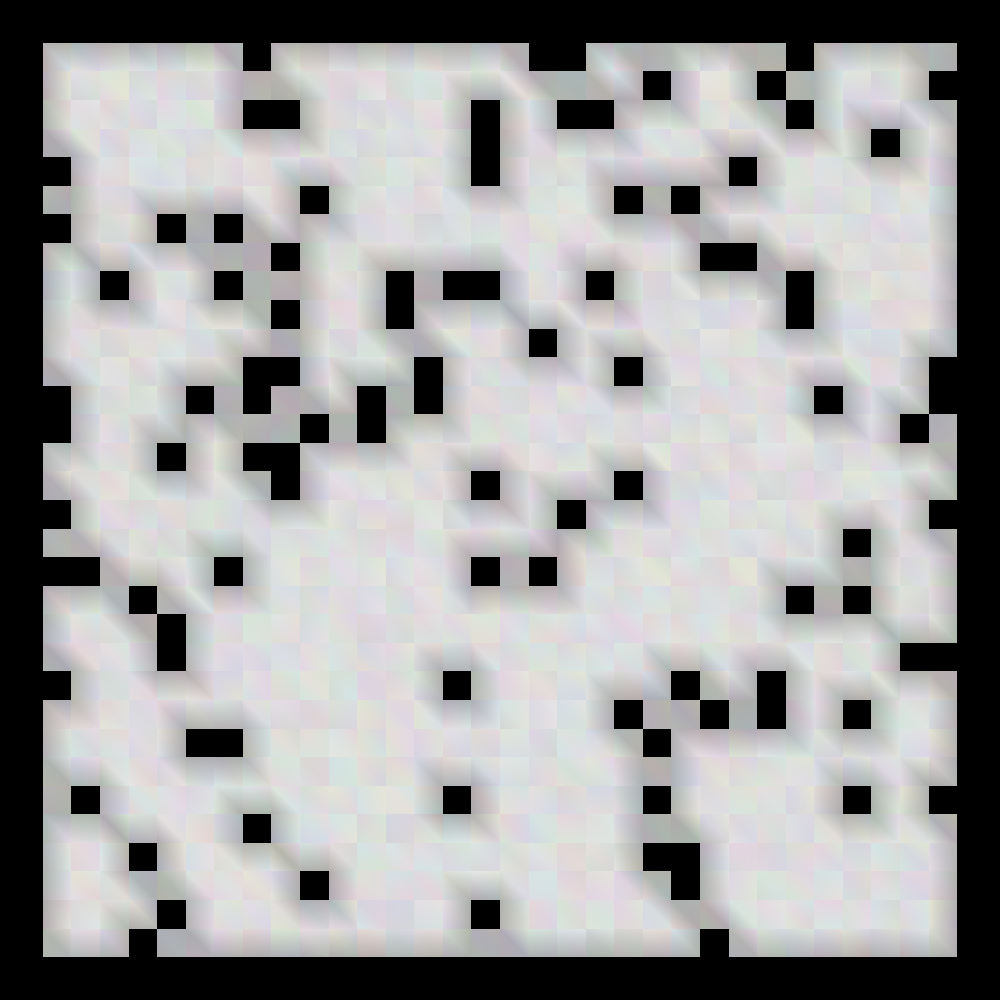}}& random-32-32-10 & directed & 12.19~$\pm$~0.82 & \textbf{8.78~$\pm$~0.78} & \textbf{1.08~$\pm$~0.29} & 7.51~$\pm$~4.23 & \textbf{7.27} & 7.09 \\
& \multirow[t]{2}{*}{\# agents: 300} & random & 15.43~$\pm$~1.12 & \textbf{14.71~$\pm$~1.10} & \textbf{1.09~$\pm$~0.29} & 1.46~$\pm$~0.90 & 7.27 & \textbf{7.30} \\
&  & speed & 10.85~$\pm$~0.89 & \textbf{8.10~$\pm$~0.86} & \textbf{1.08~$\pm$~0.29} & 10.89~$\pm$~6.16 & \textbf{7.27} & 6.28 \\
\midrule
\multirow[b!]{2}{*}{\includegraphics[width=1cm]{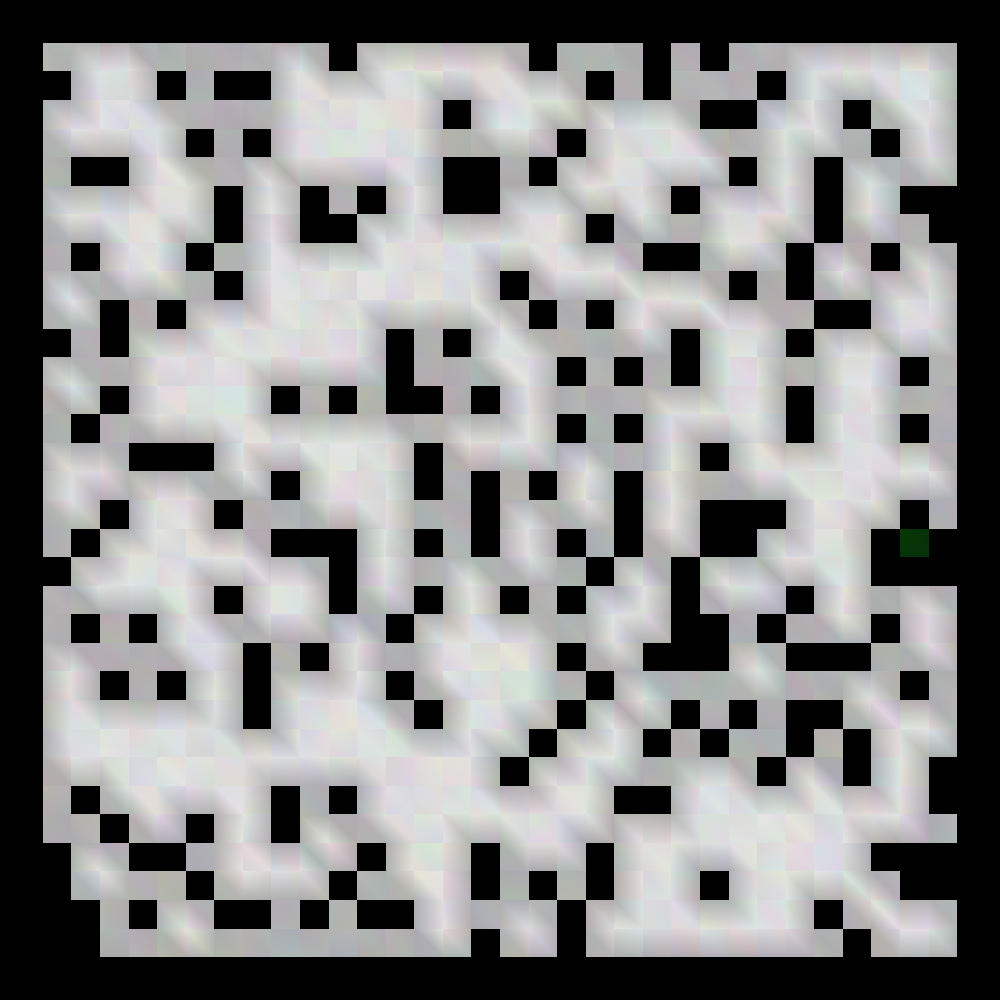}}& random-32-32-20 & directed & 10.09~$\pm$~1.02 & \textbf{6.74~$\pm$~0.72} & \textbf{0.69~$\pm$~0.31} & 1.96~$\pm$~0.93 & 4.62 & \textbf{4.63} \\
& \multirow[t]{2}{*}{\# agents: 200} & random & 12.66~$\pm$~1.08 & \textbf{12.21~$\pm$~1.02} & 0.69~$\pm$~0.31 & \textbf{0.65~$\pm$~0.26} & \textbf{4.62} & 4.57 \\
&  & speed & 9.15~$\pm$~0.97 & \textbf{6.99~$\pm$~0.85} & \textbf{0.69~$\pm$~0.31} & 2.45~$\pm$~2.14 & 4.62 & \textbf{4.64} \\

\bottomrule
\end{tabularx}
\caption{
Results from running (FA-)RHCR on different benchmark maps for a specific number of MAPF agents. The results compare the relevant metrics for RHCR with and without flow-awareness. Mean and standard deviation are computed across the different RHCR iterations. . 
}
\label{table:maptable}
\end{table*}
}

\title{Conflict Mitigation in Shared Environments \\ using Flow-Aware Multi-Agent Path Finding}
\author{Lukas Heuer${}^1$, Yufei Zhu${}^1$, Luigi Palmieri${}^2$, Andrey Rudenko${}^3$, Anna Mannucci${}^{2}$,\\ Sven Koenig${}^4$ and Martin Magnusson${}^1$ %<-this % stops a space
\thanks{\raggedright{$^{1}$L.\,Heuer, Y. Zhu and M.\,Magnusson - Örebro University
       {\tt\footnotesize  \{lukas.heuer, yufei.zhu, martin.magnusson\}@oru.se}}}% 
\thanks{\raggedright{$^{2}$L.\,Palmieri, A.\,Rudenko and A.\,Mannucci - Bosch Research
       {\tt\footnotesize  \{luigi.palmieri, anna.mannucci\}@de.bosch.com}}}%
\thanks{\raggedright{$^{3}$A.\,Rudenko - Technical University of Munich
       {\tt\footnotesize  andrey.rudenko@tum.de}}}%
\thanks{$^{4}$S. Koenig - University of California, Irvine
       {\tt\footnotesize sven.koenig@uci.edu}}%
}

\begin{document}

\maketitle

\begin{abstract}
Deploying multi-robot systems in environments shared with dynamic and uncontrollable agents presents significant challenges, especially for large robot fleets.
In such environments, individual robot operations can be delayed due to unforeseen conflicts with uncontrollable agents.
While existing research primarily focuses on preserving the completeness of Multi-Agent Path Finding (MAPF) solutions considering delays, there is limited emphasis on utilizing additional environmental information to enhance solution quality in the presence of other dynamic agents.
To this end, we propose Flow-Aware Multi-Agent Path Finding (FA-MAPF), a novel framework that integrates learned motion patterns of uncontrollable agents into centralized MAPF algorithms.
Our evaluation, conducted on a diverse set of benchmark maps with simulated uncontrollable agents and on a real-world map with recorded human trajectories, demonstrates the effectiveness of FA-MAPF compared to state-of-the-art baselines.
The experimental results show that FA-MAPF can consistently reduce conflicts with uncontrollable agents, up to 55\%, without compromising task efficiency.
\end{abstract}

\section{Introduction}\label{section:introduction}
% %Motivation
Multi-robot systems have demonstrated their effectiveness across various industries, including warehouse automation, intralogistics, and manufacturing \cite{roboticstomorrow}. 
Nevertheless, their current deployment is typically restricted to controlled environments where they can operate autonomously, isolated from human interactions and other unpredictable variables.
Recent advancements in robotics application suggest that overcoming this limitation is crucial for expanding the potential applications of multi-robot systems, allowing greater flexibility and adaptability in diverse settings \cite{proteus2024amazon}, including environments which are shared with uncontrollable agents.
Uncontrollable agents present a significant challenge when deploying multi-robot systems, as unpredictable conflicts require robots to respond in real time \cite{alkazzi2024comprehensive}. This poses a significant issue for MAPF-based coordination algorithms which typically rely on synchronous execution and generally cannot re-plan paths of individual agents separately. 

The problem of execution delays is widely acknowledged within the MAPF community, and various approaches have been developed to enhance the robustness of MAPF solutions in this regard~\cite{honig2019persistent, atzmon2020probabilistic,berndt2023receding,phan2025anytime}.
However, there is little work that investigates the mitigation of delays itself. 
We hypothesize that even a partial reduction of conflicts with uncontrollable agents can help when deploying multi-robot fleet using centralized MAPF algorithms.
While this work does not focus on improving the scalability of MAPF Algorithms, it may unlock applications which are, until now, not suited to deploy centrally coordinated multi-robot systems. 

Neither humans nor artificial dynamic agents move randomly. 
In real-world environments (shopping malls, airports, warehouses or workshops, etc.) dynamic agents will generally follow well recognizable motion patterns.
These motion patterns emerge because of the environment's topology, social norms and/or internal movement and behavior rules of the agents.
This idea is formalized with the concept of Maps of Dynamics (MoDs)~\cite{kucner2023survey}, which encode spatial or spatio-temporal motion patterns as features of the environment. 
MoDs capture structured and multimodal motion patterns and can be constructed from historical data or learned online.  
These patterns provide strong priors for predicting agent behavior into the future and over extended time horizons~\cite{zhu2023cliff}.

With the goal of improving the quality of paths generated by centralized MAPF algorithms in uncertain and dynamic environments, we propose Flow-Aware Multi-Agent Path Finding (FA-MAPF). 
The primary contribution in this context is the use of learned motion patterns of dynamic entities in a given environment, such as MoDs, in order to reduce conflicts between uncontrollable agents and the multi-robot system. 
To achieve this, we integrate multi-modal probability distributions, representing spatio-temporal motion patterns, into search-based state-of-the-art MAPF algorithms. Importantly, FA-MAPF strives for flow awareness implicitly at planning time and thus does not require ad-hoc replanning of specific agents or active sensing and detection of uncontrollable agents. Moreover, it retains most of the theoretical guarantees offered by centralized MAPF algorithms.
We employ the idea of guidance, as proposed by Zhang et al.~\citet{zhang2024guidance}, and compute edge-weights using a cost function that matches an agent’s possible actions with the local expected movement of a potential uncontrollable agent. 

Using established MAPF benchmarks~\cite{stern2019mapf}, we evaluate the effect of flow awareness on both the efficiency of MAPF algorithms and the number of conflicts with uncontrollable agents that can be expected. Furthermore, we assess the impact of flow awareness in lifelong MAPF settings in terms of runtime, throughput, and conflicts with uncontrollable agents, considering both benchmark maps and a real-world dataset. Our results indicate a clear trade-off between increased runtime and reduced conflicts with uncontrollable agents, while throughput remains largely unaffected.

\section{Preliminaries and Related Work}
In this section, we introduce the concepts that FA-MAPF is based on: (1) Multi-Agent Path Finding, (2) Maps of Dynamics and (3) Guidance, and overview the related work.
\subsection{Multi-Agent Path Finding} \label{ssec:mapf}
We adopt their notion of a \textit{classical} MAPF problem, as given by Stern et al.~\citet{stern2019mapf}, which is defined as an undirected graph $G=(V,E)$ of vertices $V$ and edges $E$, and a set of agents $K$, each associated with a start ($s_k$) and target vertex ($t_k$), i.e., $s_k, t_k \in V, ~\forall k \in K$.
Time is discretized, and there exists a set of actions $A$ from which each agent can perform one action per timestep.
An action $a \in A$ either moves the agent along an edge from its current vertex $v \in V$ to a new vertex $v'$ or causes the agent to wait in place, in which case $v=v'$. 
This relation can be formalized with a transition function $l:(V,A) \mapsto V$ such that $l(v, a) = v': v, v' \in V, a \in A$.
A path is defined as a sequence of actions $\psi = (a_0, \cdots, a_n)$. 
We call $\psi$ a valid path for agent $k$ if and only if, starting at vertex $s_k$, the execution of the actions in $\psi$ results in being in vertex $t_k$.
A solution to the MAPF problem is a set of valid, non-conflicting paths, one for each agent. This means that no agent occupies the same vertex or traverses the same edge at the same timestep.
Ma et al.~\citet{ma2017lifelong} propose a lifelong extension to the MAPF problem in which the agents are provided with a new goal anytime they reach their current one and the MAPF problem is re-instantiated after a set amount of timesteps. 

\subsubsection{Uncontrollable Agents in MAPF}
Uncontrollable or external agents that can interact with and cause delays to individual agents are a notable problem for multi-agent systems \cite{alkazzi2024comprehensive}. While there exists a lot of work in this direction for single agent path planning, to our knowledge the only work to tackle this in the context of centrally coordinated multi-agent systems is by Bonalumi et al.~\citet{bonalumi2025multi}. They propose Multi-Agent Pickup and Delivery with External Agents (MAPD-EA), including an extension to the Token Passing (TP) Algorithm \cite{ma2017lifelong} called Token Passing with collision avoidance and replanning (TP-CA). 
Unfortunately, we were not able to obtain TP-CAs implementation, which is why we can not provide a direct comparison to our method.

While we consider a similar problem formulation, FA-MAPF does not require a well-formed MAPD problem and thus integrates naturally with state-of-the-art (lifelong) MAPF solvers like Explicit Estimation Conflict Based Search (EECBS) \cite{li2021eecbs} and Rolling-Horizon Collision Resolution (RHCR) \cite{li2021lifelong}.
FA-MAPF also differs from TP-CA methodologically as we do not compute explicit occupancy likelihoods and transition probabilities, tied to the explicit detection of uncontrollable agents at planning time.
Instead, FA-MAPF utilizes MoDs to incorporate local motion patterns into the environment model, guiding the planning process to reduce the probability of conflict with uncontrollable agents.

\subsubsection{Robust MAPF}
Several approaches have been developed to enhance the robustness of MAPF schedules. 
Atzmon et al.~\citet{atzmon2020probabilistic} show how a MAPF algorithm can be made resilient to delays spanning multiple timesteps. 
Phan et al.~\citet{phan2025anytime} use agent delays as a basis for the destroy heuristic in MAPF algorithms based on Large Neighbourhood Search (LNS).
Hönig et al.~\citet{honig2019persistent} and Berndt et al.~\citet{berndt2023receding} achieve a certain degree of resilience to execution delays by capturing and enforcing the precedence relationships among robot actions of the original MAPF solution. This refinement ensures completeness and robust execution even in the presence of unknown or time-varying delays.

Unlike these works, our method does not focus on making MAPF solutions more resilient to execution delays. Instead, it aims to reduce interactions with uncontrollable agents that will cause delays, thereby reduce the required robustness of the MAPF solution in the first place.
Importantly, FA-MAPF is a complementary to robust MAPF and not substitutive.

\subsection{Maps of Dynamics}\label{ssec:mods}

Maps of Dynamics (MoDs) encode spatial and spatio-temporal motion patterns as features of the environment. In this work, we exploit the Circular-Linear Flow Field (CLiFF) Map~\cite{kucner2017enabling}, an MoD representation that models local motion patterns as continuous, multimodal distributions over velocity. CLiFF-maps support both offline construction from past observations and online lifelong updates~\cite{zhu2025cliffonline}.

CLiFF-maps are built from observations of agent motion and assign a probability distribution over velocities to each spatial location. 
These velocities are represented as direction ($\theta$) and speed ($\rho$), jointly as $\mathbf{u} = [\theta, \rho]^{\mathsf{T}}$, where $\theta \in [0,2\pi)$ and $\rho \in \mathbb{R}^+$.

To capture the joint distribution of $\theta$ and $\rho$, CLiFF-maps use \emph{Semi-Wrapped} Gaussian Mixture Models (SWGMMs), in which the circular variable $\theta$ is wrapped around the unit circle and $\rho$ is linear. Each SWGMM is composed of Semi-Wrapped Normal Distributions (SWNDs), allowing the model to represent multimodal motion patterns.

An SWND $\mathcal{N}^{SW}_{\boldsymbol{\Sigma}, \boldsymbol{\mu}}$ is formally defined as 
\begin{equation}
\mathcal{N}^{SW}_{\boldsymbol{\Sigma}, \boldsymbol{\mu}}(\mathbf{u}) = \sum_{m \in \mathbb{Z}} \mathcal{N}_{\boldsymbol{\Sigma}, \boldsymbol{\mu}}( [\theta,  \rho]^{\mathsf{T}} + 2\pi [ m,  0 ]^{\mathsf{T}} ),
\end{equation}
where $\boldsymbol{\Sigma}$ is the covariance matrix, $\boldsymbol{\mu}$ is the mean value of the directional velocity $(\theta, \rho)^{\mathsf{T}}$, and $m$ is a winding number that enumerates overlapping parts of the wrapped $\theta$ variable. 

An SWGMM is defined as a weighted sum of $J$ SWNDs: 
\begin{equation}\label{eq:swgmm}
p(\mathbf{u} | \mathbf{\xi}) = \sum_{j=1}^{J}\beta_{j}\mathcal{N}^{SW}_{\boldsymbol{\Sigma_j}, \boldsymbol{\mu_j}}(\mathbf{u}), 
\end{equation}
where $\boldsymbol{\xi} = \{ \xi_j = (\boldsymbol{\mu}_j, \boldsymbol{\Sigma}_j, \beta_j) | j \in \mathbb{Z}^+ \}$ denotes the finite set of mixture parameters, and $\beta_j$ denotes the mixing factor which satisfies $0 \leq \beta_j \leq 1$.

CLiFF-map have been applied in various domains, such as task and motion planning and trajectory prediction. Prior work by Palmieri et al.~\citet{palmieri2017kinodynamic} and Swaminathan et al.~\citet{swaminathan2018down} investigates how MoDs can be used for single robot motion planning. Furthermore, Liu et al.~\citet{liu2023human} demonstrate how human flow can be incorporated into hierarchical reinforcement learning for task planning.
Lastly, Zhu et al.~\citet{zhu2023cliff} explores how MoDs can be used to predict the motion trajectories of individual humans.

\subsection{Guidance in MAPF}
In FA-MAPF, we integrate learned motion patterns into MAPF algorithms by adapting the transition costs in the MAPF graph $G$ to capture these patterns.
This approach, recently formalized and termed \textit{guidance} by Zhang et al.~\citet{zhang2024guidance} has been comprehensively surveyed in prior work, including \citet{zang2025online, chen2024traffic, zhang2024guidance, cohen2016improved, jansen2008direction, wang2008fast}. While most existing approaches apply guidance primarily to improve the runtime efficiency or throughput of multi-agent systems, FA-MAPF distinguishes itself by focusing on minimizing interactions with uncontrollable agents. We achieve this by modeling their motion patterns and using these patterns to direct the MAPF algorithm.

To formalize the integration of guidance into FA-MAPF, we adopt the framework of a guidance graph. Specifically, we assume a directed guidance graph $G_g = (V_g, E_g, \boldsymbol{\omega})$, as described by Zhang et al.~\citet{zhang2024guidance}, where the vertex set is identical to that of the original MAPF graph (i.e., $V = V_g$) and $E_g$ includes edges representing all possible transitions available to a MAPF agent in $G$. The vector $\boldsymbol{\omega}$ encodes the weights assigned to each edge, with $\omega_e$ denoting the specific cost associated with edge $e \in E_g$. This formalism allows us to encapsulate motion patterns into the computation of $\boldsymbol{\omega}$ as explained in Section \ref{sec:method}.

\section{Flow-Aware Multi-Agent Path Finding}\label{sec:method}
This section presents the core contribution of our paper: a method for accounting of learned motion patterns of uncontrollable agents when solving a MAPF problem.

\subsection{Problem Formulation} \label{ssec:problem}

We consider a potentially lifelong MAPF problem as described in Section~\ref{ssec:mapf} to coordinate a multi-robot system.
We consider uncontrollable agents (UAs) that operate in the same environment as the agents controlled by the MAPF algorithm, i.e. MAPF agents. 
A UA is defined through its trajectory $\boldsymbol{\zeta} = (\mathbf{x}_0, \mathbf{x}_1, \dots, \mathbf{x}_\zeta)$ with Cartesian positions $\mathbf{x}_i \in \mathbb{R}^2$.

Assume two sets of trajectories of UAs: $H = (\boldsymbol{\zeta}_0, \boldsymbol{\zeta}_1, \dots, \boldsymbol{\zeta}_H)$ as an observed dataset of trajectories in the operational environment and $F= (\boldsymbol{\zeta}_0, \boldsymbol{\zeta}_1, \dots, \boldsymbol{\zeta}_F)$ as the trajectories of the UAs that are present when the MAPF solution is executed.
Importantly, we assume knowledge about $H$ and consider $F$ to be unknown, even at planning time.

A conflict between a MAPF agent and a UA occurs if, at any time, the Euclidean distance between them is smaller than the sum of their radii. 
We now aim to find a solution to the MAPF problem that minimizes conflicts between the paths in the MAPF solution and the trajectories in $F$.

\subsection{Method}
Consider an environment in Euclidean space $W$, represented by a graph $G = (V,E)$ (as described in Section~\ref{ssec:mapf}), a transformation $T_{W\mapsto V}$, and the set of possible MAPF agent actions $A$. 
A dataset $H$ of observed trajectories of uncontrollable agents is used to construct the MoD $M$, which encodes relevant motion patterns in environment $W$.

\directioncost
The core idea of FA-MAPF is to assign a cost-value to account for the alignment between an agent's action $a \in A$ with the motion pattern $M(v)$, at a given location. 
Formally, let $q: (W, A, M) \mapsto \mathbb{R}_{\geq 0}$ be a function which maps a position in the environment $W$, an action from the possible set of actions $A$, and an MoD $M$, to a cost-value. 
Using $T_{W\mapsto V}$ to map positions to vertices, we can query the MoD at the position of a vertex $v \in V$ to obtain $M(v)$, and define the flow cost $g_f(a, M(v))$ whose explicit calculation is explained in Section~\ref{subsection:cost_formulation}.
Using $g_f$ we then compute the edge weights in a guidance graph $G_g = (V_g, E_g, \boldsymbol{\omega})$ which FA-MAPF uses to plan the paths of the individual MAPF agents. 
The individual edge weights are computed as
\begin{equation} \label{eq:edgeweight}
    \omega_e = g_s + g_f\left(a, M\left(v\right)\right),~\forall~e \in E_g,
\end{equation}
where $v$ is the source vertex of $e$, $a$ is the action taken to follow the edge $e$, $M(v)$ is the SWGMM obtained from the CLiFF-map at the location $v$, and $g_s$ is the step-cost defined by the MAPF algorithm cost-function.
Fig.~\ref{fig:direction_cost} illustrates an example guidance graph on a map used in our experiments.

Because FA-MAPF does not require real-time information at planning time, the guidance graph and an admissible heuristic (including the flow cost) can be pre-computed by finding shortest paths in $G_g$ between all pairs of vertices.

\subsection{Flow Awareness}\label{subsection:cost_formulation}

To compute $g_f$, the flow cost, we use the Mahalanobis distance which measures the distance between a multivariate Gaussian distribution and a vector~\cite{mahalanobis1936generalized}. 
The Mahalanobis distance can be extended to Gaussian mixtures by computing the weighted sum over the mixture components. 
A similar formulation has been proposed for sampling-based planning~\cite{swaminathan2018down}.

Using this approach to measure the distance between an action $a$ and an SWGMM $M(v)$, obtained from a CLiFF map, requires $a$ to match the vector space of $M(v)$.
That is, $a$ should also be a velocity vector $\mathbf{u}=[\theta, \rho]^{\mathsf{T}}$.
When implementing a MAPF algorithm on a 4-connected grid, we consider $A = \{a_{\mathrm{right}}, a_{\mathrm{up}}, a_{\mathrm{left}}, a_{\mathrm{down}}, a_{\mathrm{wait}}\}$ as the set of possible actions, and a direct mapping $\alpha:A \mapsto A'$.
With $A' =\{[0, 1]^{\mathsf{T}}, [\frac{\pi}{2}, 1]^{\mathsf{T}}, [\pi, 1]^{\mathsf{T}}, [\frac{3\cdot\pi}{2}, 1]^{\mathsf{T}}, [-, 0]^{\mathsf{T}}\}$, we can use $a \in A$ and $\mathbf{a}' \in A'$ interchangeably. 

Given an SWGMM as described in Equation~\ref{eq:swgmm}, and an action $\mathbf{a}' \in A'$, we now compute
\begin{multline}\label{eq:malahanobis_distance}
    g_f\left(a, M\left(v\right)\right):=\\
    \log\gamma\sum_{j=0}^J\left(\beta_j\sqrt{\left(\boldsymbol{\mu}_j - \mathbf{a}'\right)^{\mathsf{T}}\boldsymbol{\Sigma}_j^{-1}\left(\boldsymbol{\mu}_j - \mathbf{a}'\right)}\right),
\end{multline}
where $\gamma$ is the number of UA observations made at location $v$.
Thus the sum represents the actions alignment with the underlying SWGMM and the logarithmic term scales that value with likelihood that an UA will be encountered in this location in the first place. 
Importantly, we compute the subtraction between two angles (i.e. $\boldsymbol{\mu}_\theta$ and $\mathbf{a}'_\theta$) as the shortest positive angular distance around the unit circle. 
After computing $g_f$ for all transitions in the map, we employ min-max normalization to scale the cost values to the range $[0, 1]$.

\modcost
Fig.~\ref{fig:mod_cost} exemplary illustrates a SWGMM; the way it relates to the actions $\mathbf{a}' \in A'$; and the values of $g_f$ for the corresponding actions $a \in A$.

\subsubsection*{Implementation assumptions}
Applying flow cost in MAPF is based on two key assumptions.

First, as MAPF algorithms generally work with unit time, where one action takes one timestep, we assume that agents move with a constant speed. 
In accordance with the definition of $A'$, we set the agent's speed to $1\frac{\text{m}}{\text{s}}$, which is slightly below the average human walking speed of $1.42 \frac{\text{m}}{\text{s}}$ \cite{browning2006effects}. 

Second, the wait actions $a_{\mathrm{wait}}$ lacks a defined movement angle, preventing direct computation of $g_f$ for this action. 
To resolve this, we define $g_f(a_\mathrm{wait})$ for each $M(v)$ as the mean value of $g_f$ the other four actions, with $\rho$ set to zero.

\subsection{Completeness and Optimality}
The completeness and (sub-)optimality properties of FA-MAPF depend on the underlying MAPF algorithm used. Specifically, when employing a complete and (sub-)optimal search-based MAPF algorithm, such as ECBS \cite{barer2014suboptimal}, FA-MAPF inherits these properties. Completeness is preserved because FA-MAPF only adds a finite positive value to the cost of each edge transition in the graph. When $g_f \in [0, 1]$, Equation~\ref{eq:edgeweight} shows that $g(v) < g(v') ~\forall~ v, v' \in V$. In other words, all edge costs remain positive and the cost strictly increases as the agent progresses, ensuring that the search terminates and all solutions are eventually found, that is, the search remains exhaustive and complete~\cite{russell2016artificial}.

If the underlying MAPF algorithm uses an admissible heuristic, FA-MAPF also maintains the (sub-)optimality guarantee (up to the suboptimality bound for algorithms like ECBS), though now with respect to the total cost on the guidance graph $G_g$ which includes the adapted edge weights.
However, with a simple assumption, this guarantee can be extended to only the path length of the solution.

\subsubsection*{Bounded (sub-)optimality by the optimal path length}
The cost of any path $\psi$ in a solution returned by a $\omega_1$-bounded suboptimal FA-MAPF algorithm, as described in Section~\ref{sec:method}, is composed of two terms: the path-length cost term $\textstyle\sum_{\psi}g_s$ and the flow cost term $\textstyle\sum_{\psi}g_f$ (we generally omit function arguments here for better readability).
Consider a path $\psi^*$, with cost $\textstyle\sum_{\psi^*}g_s$, to be optimal.
We prove that the cost of $\psi$, $\textstyle\sum_{\psi}g_s + \textstyle\sum_{\psi}g_f$, is bounded suboptimal with respect to $\textstyle\sum_{\psi^*}g_s$ under a simple condition.
\begin{remark}
    The path $\psi$ is $\omega_1$-bounded suboptimal w.r.t. the cost function $g_s + g_f$.
    The path $\psi'$ is optimal w.r.t. the cost function $g_s + g_f$.
    The path $\psi^*$ is optimal w.r.t. the cost function $g_s$.
\end{remark}

\begin{lemma}
    If $g_f$ is positive finite in the entire search space such that $\max_{a\in A, v \in V}(g_f(a, M(v))) \leq \omega_2~g_s$, then the cost of a $\omega_1$-bounded suboptimal path $\psi$, $\textstyle\sum_{\psi}g_s + \textstyle\sum_{\psi}g_f$, has the upper bound $(\omega_1 + \omega_1~\omega_2)\textstyle\sum_{\psi^*}g_s$ with $\omega_1 \geq 1$ and $\omega_2 \geq 0$. 
    %path-length cost $\sum g_s$ of a solution is bounded suboptimal with a factor of $\omega_1 + \omega_1~\omega_2$.
\end{lemma}

\begin{proof}
%Given an optimal FA-MAPF solution, its cost is composed of an optimal flow cost $(\textstyle\sum g_f)^*$ and an optimal path-length cost $(\textstyle\sum g_s)^*$.
Since $\psi$ is $\omega_1$ bounded suboptimal, the following relation between $\psi$ and an unknown optimal path $\psi'$ upholds:
\begin{equation}\label{eq:p1}
    \textstyle\sum_{\psi}g_s + \textstyle\sum_{\psi}g_f \leq \omega_1\left(\textstyle\sum_{\psi'}g_s + \textstyle\sum_{\psi'}g_f\right).
\end{equation}
From $\max_{a\in A, v \in V}(g_f(a, M(v))) \leq \omega_2~g_s$, we can derive the following equation, which is valid for any path $\psi^\#$: 
\begin{equation}\label{eq:p2}
    \textstyle\sum_{\psi^\#} g_f \leq \omega_2\textstyle\sum_{\psi^\#} g_s.
\end{equation}
This also includes $\psi^*$, so we can calculate the potential flow cost along $\psi^*$, and obtain
\begin{equation}\label{eq:p3}
    \textstyle\sum_{\psi^*}g_s + \sum_{\psi^*}g_f \leq \textstyle\sum_{\psi^*}g_s + \omega_2\textstyle\sum_{\psi^*}g_s.
\end{equation}
Multiplying with $\omega_1$ and rearranging yields:
\begin{equation}\label{eq:p4}
    \omega_1\left(\textstyle\sum_{\psi^*}g_s + \sum_{\psi^*}g_f\right) \leq (\omega_1 + \omega_1~\omega_2)\textstyle\sum_{\psi^*}g_s.
\end{equation}
As $\psi'$ is defined as the optimal for the cost function $g_s + g_f$, we can then write 
\begin{equation}\label{eq:p5}
    \omega_1\left(\textstyle\sum_{\psi'}g_s + \textstyle\sum_{\psi'}g_f\right) \leq \omega_1\left(\textstyle\sum_{\psi^*}g_s + \sum_{\psi^*}g_f\right).
\end{equation}
\noindent
Finally, using the transitivity of Equations~\ref{eq:p1},~\ref{eq:p5}, and~\ref{eq:p4} we obtain:
\begin{equation}\label{eq:p6}
    \textstyle\sum_{\psi}g_s + \textstyle\sum_{\psi}g_f \leq (\omega_1 + \omega_1~\omega_2)\textstyle\sum_{\psi^*}g_s.
\end{equation}
Thus, the cost of the $\omega_1$-bounded suboptimal path in the solution of a FA-MAPF problem has the upper bound $(\omega_1 + \omega_1\omega_2)\textstyle\sum_{\psi^*}g_s$. Summing over all paths extends this to the entire FA-MAPF solution.
\end{proof}

\modareas
\section{Evaluation}\label{sec:evaluation}
As our main objective is to reduce conflicts with uncontrollable agents (UAs), we focus our evaluation on how much such reduction we can expect on different maps and different types of motion, and how much it impacts computation runtime and task efficiency.

\subsubsection{Experiments}
In our first experiment, we use EECBS, with a suboptimality factor of $1.2$, to solve a standard MAPF problem. 
We specifically want to investigate the computational efficiency of FA-MAPF and use maps from the established benchmark by Stern et al.~\citet{stern2019mapf}.

The second experiment solves a lifelong MAPF problem using RHCR with ECBS, with a suboptimality factor of $1.5$, as a high-level solver and SIPP~\cite{phillips2011sipp} as a low-level solver. 
We refer to the flow-aware version of RHCR as FA-RHCR.
This experiment focuses on showing how much reduction in conflicts we can expect as well as the tradeoff in terms of runtime and efficiency. 
If not stated differently, RHCR is run with a simulation time, replanning window and conflict resolution horizon of $2000$, $20$, and $40$ timesteps respectively. 
The tasks for each agent are chosen randomly and to ensure reproducibility across methods, we use a predetermined task queue per map and agent.

All experiments are run on a Intel i7-12700K CPU with 32GB of RAM.

\subsubsection{Metrics}
In our experiment we evaluate several metrics. \textit{solved} is specific to the first experiment (Fig.~\ref{fig:oneshot}) and refers to the percentage of instances solved out of the 25 scenarios contained in the benchmark, within a time-limit of 5 seconds. 
Since the UAs are not directly considered at planning time they can not cause a MAPF problem instance to become infeasible. Therefore a reduction in the \textit{solved} metric is always attributed to a worse runtime-performance and not collisions between UAs and MAPF agents.

We evaluate algorithm \textit{runtime} directly for both experiments. In the first experiment it refers to the time EECBS takes to find a solution. In the second experiment it refers to the average solving time per RHCR iteration. 
To asses FA-MAPF's impact on the efficiency we provide \textit{throughput} as the average number of tasks completed per timestep. 
\textit{UA Conflicts} is the average number of conflicts between MAPF agents and UAs per timestep. 
As described in Section \ref{ssec:problem} a conflict occurs if, at any one time, the Euclidean distance between a MAPF agent and a UA is less than the sum of their radii (each radii set to $0.3$~m). 

\oneshot
\maptable

\subsubsection{Maps}
We perform experiments on 8 different maps, namely \textit{empty-32-32}, \textit{random-32-32-10}, \textit{random-32-32-20}, \textit{maze-32-32-2}, \textit{maze-32-32-4}, \textit{ht\_chantry}, \textit{den312d} and \textit{warehouse-10-20-10-2-1} \cite{stern2019mapf}.
In addition we run an experiment on the ATC dataset map~\cite{brvsvcic2013person} which allows us to test our approach against real-world recorded human motion data. On all maps, one grid-cell corresponds to $1 \times 1$ meters.

\subsubsection{Uncontrollable Agents}
We simulate UAs on the benchmark maps using SIPP with a branching factor of 5 to allow diagonal movement. The UAs move at a constant speed of $1\frac{\text{m}}{\text{s}}$ and to not consider inter-agent collisions.

We hypothesize that the way UAs move affects how well MoDs can help in avoiding conflicts. For instance agents moving randomly in the environment are much harder to synthesize into specific movement patterns then agents that stream between to specific map regions. 
We consider 3 types of UA movement, defined by how the start and goal positions of the UAs are chosen:
\textit{random} refers to the start and goal positions being chosen randomly on the map.
\textit{directed} refers to the start and goal location being sampled from dedicated areas such that clear movement patterns are enforced. 
\textit{speed} refers to the start and goal locations being sampled from dedicated areas, with some areas being associated with twice the movement speed of others.
For each map and movement type we sample 10,000 trajectories, which are used to compute the MoDs. 

For the first experiment, on each map, we consider 100 UAs starting at $t=0$ and moving towards their respective goal states, vanishing upon arrival.
For the second experiment we consider one UA to appear every timestep, follows a trajectory based on the movement type, and vanish at its goal location. This ensures there are UAs moving continuously through the environment.

Fig.~\ref{fig:mod_areas} shows the start and goal areas on the \textit{den312d}-map for \textit{directed} UA movement, along with its CLiFF-map visualization.

\section{Results and Discussion}
In the following we present our experimental results. 
\subsubsection{Experiment 1}
Fig.~\ref{fig:oneshot} shows the results obtained from the first experiment, where we evaluate the scalability of FA-MAPF by solving one-shot EECBS on the different benchmark maps.
MAPF without flow cost is more computationally efficient and typically maintains 100\% solved rate (within 5 s) for larger numbers of agents compared to FA-MAPF. Equivalently, the runtime starts to increase later.
However, the extent to which FA-MAPF scales worse, seems to depend on the type of map. That \textit{den312d} and \textit{ht\_chantry} scale more poorly than the random maps suggests that on maps with more \textit{structure} the guiding effect of FA-MAPF is stronger. The movement type of the UAs does not seem to have a general impact on the scalability in this case.

\numagentcomp
\windowscomp
\atcresults

\subsubsection{Experiment 2}
Table~\ref{table:maptable} presents results from the second experiment using RHCR.
The results show that FA-MAPF reduces \textit{UA conflicts} by up to 55\% compared to the baseline, in exchange for an increase in runtime.
As in Experiment 1, the increased \textit{runtime} is due to the non-uniform transition cost in the guidance graph, which causes the low-level search algorithm to expand more nodes and thus taking longer to find the final (sub-)optimal path for a MAPF agent. 
While \textit{throughput} is generally an important metric, it is not notably effected by using FA-MAPF.
Comparing the results across the different maps suggests that FA-MAPF excels on maps with a room-like structure (i.e. \textit{den312d} and \textit{ht\_chantry}) but is less suited for maze-like environments. 
We believe this is because, on maps with rooms, flow patterns emerge more clearly and it is easier for the MAPF agents to adapt their paths accordingly. In contrast, maze-like maps do not provide sufficient room to maneuver, limiting ability of MAPF agents to adjust to the motion patterns of the UAs.

We want to highlight that the \textit{runtime} in Fig.~\ref{fig:oneshot} and Table~\ref{table:maptable} are not directly comparable. 
As the first experiment uses one-shot EECBS and reports the number of unsolved instances, unsolved instances where the runtime limit is reached do not contribute to the metric calculation. This means that the reported values represent the runtime for the instances where a solution was eventually found within 5 seconds. 
In contrast to EECBS, RHCR handles infeasible iterations differently. 
Valid paths for the agents are still used, and only ones that are conflicting are frozen for the respective planning window. 
Thus infeasible iterations in RHCR are considered for the metric calculation, including the runtime which contributes to these iterations with its maximum value.

\subsubsection{Additional Evaluations}
Fig.~\ref{fig:num_agent_comp} and \ref{fig:windows_comp} show results on the \textit{den312d} map with \textit{directed} UA movement, varying the number of agents and the replanning horizon. %for a varying number of agents and replanning-horizon respectively. 
It is worth noting that, for larger numbers of agents, FA-MAPF seems to dampen the drop-off in \textit{throughput}. 
We cannot make a definitive claim as to why this is, but it suggest that FA-MAPF produces more feasible paths and live agents even in situations where a full solution can not be found in time.
\warehouse
\warehousecc

Importantly, we also evaluate how FA-MAPF works with real-world human motion data. We run RHCR with 40 MAPF agents for 33000 timesteps, i.e. 9 hours, on a grid-map representing the ATC shopping mall. For evaluation we replay the first day of the ATC-dataset~\cite{brvsvcic2013person} to compute conflicts of the MAPF solutions with the trajectories in the dataset.
Fig.~\ref{fig:atc_results} shows that in real-world scenarios, FA-MAPF is able to consistently reduce \textit{UA conflicts} by a notable margin, at the cost of an increase in runtime.

We also investigate how FA-MAPF works in warehouse-like environments. 
In order to do so, we consider two settings, whose evaluation confirms the strengths and weaknesses of FA-MAPF we identified in the other experiments.
Firstly, Fig.~\ref{fig:warehouse} shows results where UAs use the shortest path to their goal without restriction. This causes corridors to be traversed in both directions and no clear flow-pattern to emerge, hence FA-MAPF is unable to improve \textit{UA conflicts} in this case.
In a second case shown in Fig.~\ref{fig:warehouse_crisscross}, we impose highway patterns \cite{li2023study}, only allowing the UAs to traverse corridors in one direction. The allowed directions alternate from top to bottom and left to right.
The results show that FA-MAPF is successfully able to infer the highway patterns from the MoD, significantly reducing \textit{UA conflicts.}

\section{Conclusion and Future Research}\label{section:conclusion}
We propose a novel method, Flow-Aware MAPF, to incorporate statistical learned information about motion patterns of uncontrollable agents into centralized MAPF.
Our evaluation demonstrates that FA-MAPF significantly reduces conflicts with uncontrollable agents, reducing the need for low-level collision avoidance mechanisms to repair local paths.

The concept of FA-MAPF establishes a foundation for several research directions:
\textbf{(1)} Integrating flow awareness into large-scale MAPF algorithms such as PIBT \cite{okumura2022priority} or LaCAM \cite{okumura2023lacam}.
\textbf{(2)} Investigating the impact of conflicts with uncontrollable agents on efficiency metrics of a coordinated multi-robot system.
\textbf{(3)} Extending FA-MAPF to incorperate time-dependent Maps of Dynamics \cite{zhu2025cliffonline}.

In conclusion, this paper lays groundwork for transitioning multi-robot systems from fully controlled spaces with no external disturbances to more challenging and interactive environments.

\bibliography{aaai2026}

\end{document}